\documentclass{article} 
\usepackage{iclr2026_conference,times}
\usepackage{algorithm, algorithmic, amsmath, amsthm, amssymb}

\newtheorem{theorem}{Theorem}[section] 
\newtheorem{lemma}[theorem]{Lemma}     

\newtheorem{assumption}[theorem]{Assumption}
\newtheorem{remark}[theorem]{Remark}

\usepackage{amsmath,amsfonts,bm}

\def\eqref#1{equation~\ref{#1}}

\def\1{\bm{1}}

\DeclareMathAlphabet{\mathsfit}{\encodingdefault}{\sfdefault}{m}{sl}
\SetMathAlphabet{\mathsfit}{bold}{\encodingdefault}{\sfdefault}{bx}{n}

\usepackage{hyperref}
\usepackage{url}
\usepackage{graphicx}
\usepackage{booktabs}

\title{When Sensors Fail: Temporal Sequence Models for Robust PPO under Sensor Drift}

\author{Kevin Vogt-Lowell, Theodoros Tsiligkaridis, Rodney Lafuente-Mercado, Surabhi Ghatti \\
MIT Lincoln Laboratory \\
\And
Shanghua Gao, Marinka Zitnik \\
Harvard \\
\And
Daniela Rus \\
MIT
}

\iclrfinalcopy 
\begin{document}

\maketitle

\begin{abstract}
Real-world reinforcement learning systems must operate under distributional drift in their observation streams, yet most policy architectures implicitly assume fully observed and noise-free states. We study robustness of Proximal Policy Optimization (PPO) under temporally persistent sensor failures that induce partial observability and representation shift. To respond to this drift, we augment PPO with temporal sequence models, including Transformers and State Space Models (SSMs), to enable policies to infer missing information from history and maintain performance. Under a stochastic sensor failure process, we prove a high-probability bound on infinite-horizon reward degradation that quantifies how robustness depends on policy smoothness and failure persistence. Empirically, on MuJoCo continuous-control benchmarks with severe sensor dropout, we show Transformer-based sequence policies substantially outperform MLP, RNN, and SSM baselines in robustness, maintaining high returns even when large fractions of sensors are unavailable. These results demonstrate that temporal sequence reasoning provides a principled and practical mechanism for reliable operation under observation drift caused by sensor unreliability.
\end{abstract}

\section{Introduction}
Real-world reinforcement learning (RL) systems, from robotic control to autonomous driving, depend on sensor feedback that is often unreliable. Failures, communication dropouts, or transient corruption lead to partial observability and degraded performance \citep{jaakkola_reinforcement_1994, kaelbling_planning_1998}. Standard RL agents, especially those based on multilayer perceptrons (MLPs), assume fully observed states and thus suffer sharp reward losses when inputs become unreliable \citep{morad2023popgymbenchmarkingpartiallyobservable, pleines2023memory}.

In practical systems, sensor outages exhibit both temporal persistence and correlations between related components \citep{vuran_spatio-temporal_2004, das2016}. To capture these effects during analysis and evaluation, we model sensor reliability using a standard stochastic failure process that accounts for individual- and group-level dependencies. This framework allows for a systematic study of robustness without introducing new assumptions about failure dynamics.

Building on this setting, our focus is on robust sequence-based PPO agents that can leverage temporal structure to infer missing information and maintain performance under partial observability. Our key contributions are as follows.
\begin{itemize}
    \item \textit{Sequence-Model PPO Architectures.} We integrate Transformer- and State Space–based encoders into PPO, enabling agents to exploit temporal dependencies to enhance decision-making.
    \item \textit{Theoretical Robustness Analysis.} We derive a high-probability bound on infinite-horizon reward degradation under stochastic observation failures, quantifying how robustness scales with policy smoothness and failure persistence.
    \item \textit{Empirical Evaluation.} We show that Transformer PPO agents achieve substantially higher reward retention under severe sensor dropout compared to MLP, RNN, and SSM baselines across MuJoCo tasks.
\end{itemize}

Together, these results establish sequence modeling as a key mechanism for robustness in online RL, demonstrating how temporal reasoning mitigates the brittleness of standard policy architectures in realistic, unreliable environments.

\section{Related Work}
\textbf{Partial Observability in RL} In the context of addressing partial observability in reinforcement learning, early efforts focused on adding memory through recurrence to value‐based agents. Deep Recurrent Q‐Networks (DRQN) \citep{hausknecht2015} have been proposed which replace DQN’s first fully-connected layer with an LSTM to integrate information over time from single-frame inputs. DRQN not only matches DQN’s performance on fully observable Atari games but also degrades more gracefully when evaluated under partially observed conditions. \cite{zambaldi2019} introduced a relational module within a model-free RL agent that encodes observations as sets of entity vectors and applies iterative message-passing (self-attention) to reason over object relations, yielding substantial gains in sample efficiency, generalization to novel scenarios, and interpretability on tasks like StarCraft II and Box-World. RLBenchNet \citep{smirnov2025rlbenchnet} provides an empirical comparison of sequence models under PPO on control and memory-based environments. However, this work is purely empirical and does not offer a theoretical characterization of performance degradation under partial observability. Moreover, the masking mechanisms used in their partially observable settings are simplistic (e.g., permanently removing velocity components or shrinking observation windows) and do not model realistic sensor failures with temporal persistence, correlation, or distributional drift. As a result, these benchmarks do not capture the structured, time-correlated observation failures encountered in real-world systems.

\textbf{Transformer Models} The Transformer architecture, introduced by \cite{Vaswani:2017}, revolutionized sequence modeling by relying entirely on self-attention mechanisms, dispensing with recurrence and convolution. Its ability to model long-range dependencies with parallelizable computation has made it the foundation of numerous advances in natural language processing and beyond \citep{devlin2018, brown2020}. Modified Transformers, such as UniTS \citep{gao2024units} and Transformer-XL \citep{dai2019}, have pushed Transformers even further by enabling universal time series representations and dependency learning beyond fixed-size context windows, respectively.

\textbf{State Space Models} Recent advances in sequence modeling have positioned structured state space models (SSMs) as an alternative to Transformer attention for capturing long-range dependencies with favorable scaling. S4 introduced the “structured state space sequence” formulation, showing that carefully parameterized continuous-time linear systems can be made computationally efficient while retaining strong, long-context performance \citep{Gu:2022}. Building toward simpler and more scalable SSM/RNN hybrids, Linear Recurrent Units (LRUs) “resurrect” efficient recurrent-style sequence modeling by using a structured (diagonal-style) recurrence that supports fast parallel training while maintaining strong long-range capability \citep{orvieto2023}. Mamba further advanced this line by introducing selective (input-dependent) state space updates, enabling token-wise, content-adaptive computation while preserving linear-time scaling in recurrent mode \citep{mamba2024}. Most recently, LinOSS \citep{linoss} proposed an oscillatory SSM derived from stable discretizations of forced second-order (harmonic oscillator) dynamics, yielding stable long-horizon behavior under a particularly simple condition—a nonnegative diagonal state matrix—and integrating efficiently via associative parallel scans.


\section{Sensor Failure Model}
We model correlated observation failures using a two-layer Markov process that captures both individual and group-level reliability dynamics. Each sensor follows a binary Markov chain for local hardware reliability, while sensor groups share a higher-level process representing subsystem dependencies (e.g., shared communication buses or power lines). This structure captures temporal persistence and spatial correlation while remaining analytically tractable with closed-form steady-state probabilities~\citep{huang2007, liu2006, mo2013}.

For sensor $i$, let $z_i(t) \in \{0,1\}$ denote its operational status, evolving as a two-state Markov chain with failure probability $p_{\mathrm{fail}}$ and recovery probability $p_{\mathrm{recover}}$. The steady-state probability of operation is $\pi_z = p_{\mathrm{recover}}/(p_{\mathrm{fail}} + p_{\mathrm{recover}})$.
Each group $j$ has a binary variable $y_j(t) \in \{0,1\}$ with analogous dynamics governed by $p^{\mathrm{group}}_{\mathrm{fail}}$ and $p^{\mathrm{group}}_{\mathrm{recover}}$, yielding steady-state $\pi_y = p^{\mathrm{group}}_{\mathrm{recover}}/(p^{\mathrm{group}}_{\mathrm{fail}} + p^{\mathrm{group}}_{\mathrm{recover}})$.
For sensor $i$ in group $j$, the effective operational status is $x_i(t) = z_i(t) \cdot y_j(t)$, requiring both individual and group processes to be up. Under independence, the effective steady-state probability is $\pi_x = \pi_z \cdot \pi_y$. The effective failure and recovery probabilities are:
\begin{align}
p_{\mathrm{fail}}^{\mathrm{eff}} &= 1 - (1 - p_{\mathrm{fail}})(1 - p^{\mathrm{group}}_{\mathrm{fail}}), \label{eq:peff}\\
p_{\mathrm{recover}}^{\mathrm{eff}} &\approx p_{\mathrm{recover}} \cdot p^{\mathrm{group}}_{\mathrm{recover}}. \label{eq:precover_eff}
\end{align}
A wide array of failure dynamics can be simulated with this model, including fast individual failures, fast group failures, mixed dynamics, and slow recovery with prolonged outages.

\section{Sequence-based PPO Agents}
\paragraph{Motivation.}
Standard PPO agents often use a feed-forward MLP policy $\pi_\theta(a_t \mid s_t)$ that maps the \emph{current} observation/state to an action.
In partially observable settings or under sensor dropouts, conditioning on only $s_t$ can induce brittle behavior because the policy ignores temporal context.
We therefore couple PPO with sequence encoders that summarize recent history (Transformers) or maintain a recurrent hidden state (RNNs/SSMs).

\subsection{Transformer-based PPO agent}
\label{sec:transformer_ppo}

\paragraph{History buffer.}
For each parallel environment, we maintain a circular buffer of the most recent $L$ observations 
$\mathcal{B}_t = (o_{t-L+1}, \dots, o_t)$, where $L$ is the configured sequence length.
At time $t$, we form a length-$L$ sequence $X_t \in \mathbb{R}^{L \times d_{\text{in}}}$ by rolling the buffer so that the oldest valid element appears first, and we construct a padding mask $M_t \in \{0,1\}^{L}$ indicating which positions are invalid (e.g., before the buffer is filled or after an environment resets).

\paragraph{Encoder.}
We first project each observation to the model dimension $d$ and add \emph{sinusoidal} positional encodings:
\begin{equation}
\tilde{X}_t = \text{PE}\!\left(X_t W_{\text{in}} + b_{\text{in}}\right), \qquad
\tilde{X}_t \in \mathbb{R}^{L \times d}.
\end{equation}
We then apply a Transformer encoder with temporal self-attention using the key-padding mask $M_t$:
\begin{equation}
H_t = \text{TransformerEnc}(\tilde{X}_t; M_t), \qquad
H_t = (h_{t,1}, \dots, h_{t,L}),\; h_{t,i}\in\mathbb{R}^{d}.
\end{equation}

\paragraph{Attention pooling.}
To obtain a fixed-size feature vector that can be fed into separate actor and critic heads, we apply learned attention pooling over time:
\begin{equation}
e_{t,i} = w^\top h_{t,i} + b, \qquad
\alpha_{t,i} = \frac{\exp(e_{t,i})}{\sum_{j:\,M_{t,j}=0}\exp(e_{t,j})}, \qquad
z^{\text{attn}}_t = \sum_{i:\,M_{t,i}=0} \alpha_{t,i} h_{t,i}.
\end{equation}


\subsection{RNN/SSM-based PPO Agent}
\label{sec:rnn_ssm_ppo}

\paragraph{Overview.}
We unify recurrent neural networks (RNNs; e.g., GRU/LSTM) and recurrent state-space models (SSMs; e.g., LRU) 
under a common \emph{recurrent latent-state encoder} that maintains memory across time.
At each step, the policy conditions on a contextual feature $z_t$ computed from the current input and a carried hidden state $h_{t-1}$.
This allows the PPO policy to use temporal context under partial observability while keeping the PPO objective unchanged.

\paragraph{Generic recurrent encoder.}
Let $o_t$ denote the observation and $x_t \in \mathbb{R}^{d}$ represent an input embedding:
\begin{equation}
x_t = g_{\text{enc}}(o_t).
\end{equation}
We model both RNNs and SSMs as a stateful mapping
\begin{equation}
(h_t,\, z_t) = \mathcal{E}_\psi(h_{t-1}, x_t; d_t),
\label{eq:generic_encoder}
\end{equation}
where $h_t$ is the recurrent memory state, $z_t$ is the emitted feature, and $d_t\in\{0,1\}$ is the episodic done flag. Similarly to the Transformer-based agent, the feature vector $z_t$ is then passed to separate actor and critic heads.

\section{Theory}
We prove a high-probability bound on the infinite-horizon reward degradation for the RL agent under the stochastic sensor failure model.

The notation is as follows. \(\mathcal{M}=(\mathcal S,\mathcal A,P,r,\gamma)\) is a discounted Markov decision process (MDP) with discount \(0<\gamma<1\). The random one-step reward at $(s,a)$ is $r(s,a)$ and its expectation is
$R(s,a):=\mathbb E[r(s,a)]$. A stochastic policy $\pi$ maps observations to action distributions.
Given an observation map $h:\mathcal S\to\mathbb R^d$, actions are drawn as
$a_t\sim \pi(\cdot \mid h(S_t))$.
The \emph{action–value function} (or $Q$-function) of $\pi$ is
\begin{equation}
\label{eq:def-Qpi}
Q^\pi(s,a)
:= \mathbb{E}_\pi\!\left[
\sum_{t=0}^{\infty} \gamma^{t}\, r(S_t,a_t)
\;\middle|\;
S_0=s,\ a_0=a,\ S_{t+1}\sim P(\cdot\mid S_t,a_t),\ a_{t\ge 1}\sim \pi(\cdot\mid h(S_t))
\right].
\end{equation}
Equivalently, $Q^\pi$ satisfies the Bellman expectation equation
\begin{equation}
\label{eq:bellman-Q}
Q^\pi(s,a)
= R(s,a) + \gamma\,\mathbb E_{S'\sim P(\cdot\mid s,a)}
\Big[\, \mathbb E_{a'\sim \pi(\cdot\mid h(S'))}\big[\,Q^\pi(S',a')\,\big] \Big].
\end{equation}
The \emph{state value function} is
\begin{equation}
\label{eq:def-Vpi}
V^\pi(s)
:= \mathbb{E}_\pi\!\left[
\sum_{t=0}^{\infty} \gamma^{t}\, r(S_t,a_t)
\;\middle|\; S_0=s
\right]
= \mathbb E_{a\sim \pi(\cdot\mid h(s))}\!\left[\,Q^\pi(s,a)\,\right],
\end{equation}
and the \emph{advantage} is $A^\pi(s,a):=Q^\pi(s,a)-V^\pi(s)$.

A mask \(m_t\in\{0,1\}^d\) zeroes failed sensors: \(h_{m_t}(s_t)=m_t\odot h(s_t)\). The mask process \(\{M_t\equiv m_t\}_{t\ge0}\) is generated by the two-layer Markov model (per-sensor chains and group chains). The stationary marginal probability that sensor \(i\) is ``up'' is denoted \(\pi_{x,i}\) (for identical sensors we write \(\pi_x\)). For the two-layer model this marginal is the product formula: \(\pi_x=\dfrac{p_{\mathrm{recover}}}{p_{\mathrm{fail}}+p_{\mathrm{recover}}}\cdot \dfrac{p_{\mathrm{grouprecover}}}{p_{\mathrm{groupfail}}+p_{\mathrm{grouprecover}}}.\)

We make the following assumptions.
\begin{assumption}[Bounded sensor outputs]\label{assump:1:bounded}
For all $s\in\mathcal S$ and $i\in\{1,\dots,d\}$, $|h_i(s)|\le B_i<\infty$.
\end{assumption}

\begin{assumption}[Policy smoothness — Wasserstein Lipschitzness]\label{assump:2:wasserstein}
Let $(\mathcal A,\|\cdot\|)$ be the action space with the metric inducing the $1$-Wasserstein distance $W_1$ on $\mathcal P(\mathcal A)$. There exists $L_\pi\ge 0$ such that for all observations $o,o'\in\mathbb R^d$,
\[
W_1\!\big(\pi(\cdot\mid o),\,\pi(\cdot\mid o')\big)\ \le\ L_\pi\,\|o-o'\|_1.
\]
\end{assumption}

\begin{assumption}[$Q^\pi$ is Lipschitz in action]\label{assump:3:Qlip}
There exists $L_Q\ge 0$ such that for all $s\in\mathcal S$ and $a,a'\in\mathcal A$,
\[
\big|Q^\pi(s,a)-Q^\pi(s,a')\big|\ \le\ L_Q\,\|a-a'\|.
\]
\end{assumption}

\begin{assumption}[Augmented-chain geometric ergodicity]\label{assump:4:augmix}
Fix the policy $\pi$. The augmented process $\Xi_t=(S_t,M_t)$ is an irreducible, aperiodic Markov chain on $\mathcal S\times\{0,1\}^d$ with stationary law $\pi_\Xi$, and is geometrically ergodic. Write its total-variation mixing time at tolerance $1/8$ as
\[
\tau\ :=\ t_{\mathrm{mix}}(1/8)\ :=\ \inf\Big\{t\ge 0:\ \sup_{\xi_0}\big\|\mathcal L(\Xi_t\!\mid\Xi_0=\xi_0)-\pi_\Xi\big\|_{\mathrm{TV}}\le\tfrac{1}{8}\Big\}.
\]
\end{assumption}

\begin{assumption}[Mask exogeneity / independence]\label{assump:5:indep}
Under the stationary law of $\Xi_t$, $M_t$ is independent of $S_t$ and has stationary distribution $\pi_M$ with marginals $\mathbb P(M_{t,i}=1)=\pi_{x,i}$. Consequently, for all $i$,
\[
\mathbb E\big[(1-M_{t,i})\,|h_i(S_t)|\big]\ =\ (1-\pi_{x,i})\,\mathbb E_{S\sim d_\pi}\big[|h_i(S)|\big]\ :=\ (1-\pi_{x,i})\,h_i,
\]
where $d_\pi$ is the stationary state distribution under $\pi$.
\end{assumption}

\paragraph{Loss variables.}
Define the \emph{one-step counterfactual gap in future return} at time $t$ as
\[
\Delta_t\ :=\ \mathbb E_{a\sim\pi(\cdot\mid h(S_t))}\!\big[Q^\pi(S_t,a)\big]
\;-\;
\mathbb E_{a\sim\pi(\cdot\mid h_{M_t}(S_t))}\!\big[Q^\pi(S_t,a)\big],
\]
its nonnegative version $X_t:=|\Delta_t|$, and the discounted cumulative loss
\[
S\ :=\ \sum_{t=0}^\infty \gamma^t X_t,\qquad \mu_S:=\mathbb E[S].
\]
Let $C_{\max}\ :=\ L_Q\,L_\pi\sum_{i=1}^d B_i$. $\Delta_t$ measures, in units of discounted \emph{future} return, how much is lost at time $t$ by drawing the action from the masked-action distribution $\pi(\cdot\mid h_{M_t}(S_t))$ instead of the full-observation distribution $\pi(\cdot\mid h(S_t))$, while evaluating consequences with the same $Q^\pi$.

\begin{theorem}[High-probability reward-degradation bound]\label{thm:main}
Assume \ref{assump:1:bounded}–\ref{assump:5:indep}. Fix $\delta\in(0,1)$. Then, with probability at least $1-\delta$,
\[
S\ \le\ \mu_S\ +\ C_{\max} \min\left\{ \sqrt{\frac{2\tau}{1-\gamma^2}\,\ln\frac{2}{\delta}}
\ +\ \frac{4}{3}\,\tau\,\ln\frac{2}{\delta}, \frac{1}{1-\gamma} \right\}.
\]
Moreover, the mean satisfies
\[
\mu_S\ \le\ \frac{L_Q L_\pi}{1-\gamma}\sum_{i=1}^d (1-\pi_{x,i})\,h_i,
\qquad h_i:=\mathbb E_{S\sim d_\pi}\big[|h_i(S)|\big].
\]
\end{theorem}

\paragraph{Interpretation of the bound.} 
The mean term $\mu_S$ captures the average loss and scales linearly with each sensor's down-time $(1-\pi_{x,i})$ and with the policy and critic sensitivities $L_\pi$ and $L_Q$. Notably, only the marginal up-rates $\{\pi_{x,i}\}$ enter this term, so correlations between sensors do not directly affect the expected degradation. The stochastic deviation around this mean has two components: a square-root term $C_{\max}\sqrt{\tfrac{2\tau}{1-\gamma^2}\ln\frac{2}{\delta}}$ that grows with slower mixing (larger $\tau$), longer effective horizon (larger $\gamma$), and desired confidence (smaller $\delta$), and a linear correction $\frac{4}{3}C_{\max}\tau\ln\frac{2}{\delta}$ from the Bernstein inequality that is typically dominated by the square-root term unless $\delta$ is extremely small or $\tau$ is large. The leverage constant $C_{\max}=L_Q L_\pi \sum_i B_i$ represents the worst-case per-step impact, factoring out policy smoothness, critic smoothness, and observation scale, and thus globally scales both deviation terms.

\paragraph{Dependence on the mask process.} 
In the two-layer mask model, the stationary per-sensor up-rate factorizes as $\pi_x = \pi_z\,\pi_y$ where $\pi_z = \frac{p_{\mathrm{recover}}}{p_{\mathrm{fail}}+p_{\mathrm{recover}}}$ and $\pi_y = \frac{p_{\mathrm{group\,recover}}}{p_{\mathrm{group\,fail}}+p_{\mathrm{group\,recover}}}$; higher recovery rates or lower failure rates increase $\pi_x$ and reduce the mean term $\mu_S$. The burstiness of outages, quantified by the mixing time $\tau$, can be bounded conservatively using the spectral gaps of the sensor and group layers: $\tau \lesssim \frac{\ln 4}{\min\{g_{\mathrm{sensor}},\,g_{\mathrm{group}}\}}$ where $g_{\mathrm{sensor}} := p_{\mathrm{fail}}+p_{\mathrm{recover}}$ and $g_{\mathrm{group}} := p_{\mathrm{group\,fail}}+p_{\mathrm{group\,recover}}$. Slower group or sensor dynamics (smaller spectral gap) increase $\tau$, which widens both the square-root and linear deviation terms in the high-probability bound.

\section{Empirical Results}
Our experimental RL environments are based on MuJoCo \citep{todorov2012mujoco}, an RL  physics engine for simulating continuous control tasks involving complex, articulated robots. These environments focus on locomotion, balance, and agility, providing rich benchmarks for developing and testing RL algorithms in high-dimensional, continuous control tasks. 

We experiment on four standard MuJoCo continuous-control benchmarks: HalfCheetah-v4, Hopper-v4, Walker2d-v4, and Ant-v4. We train eight PPO-based agents under full observability and partial observability induced by our sensor failure model. The suite of models tested comprised an MLP baseline and the following sequence-based models:
\begin{itemize}
    \item \textbf{RNNs/SSMs}: Linear Recurrent Unit (LRU) \citep{pmlr-v202-orvieto23a} and Gated Recurrent Unit (GRU) \citep{cho2014} 
    \item \textbf{Transformers}: Transformer \citep{Vaswani:2017}, UniTS \citep{gao2024units}, and Gated Transformer-XL (GTrXL) \citep{dai2019}
\end{itemize}

For our sensor model, we partition observations into three sensor groups with failure and recover probabilities $p_{\mathrm{fail}}$ = 1\%, $p_{\mathrm{recover}}$ = 90\%, $p_{\mathrm{fail}}^{\mathrm{group}}$ = 55\%, and $p_{\mathrm{recover}}^{\mathrm{group}}$ = 90\%, yielding an effective recovery rate $p_{\mathrm{recover}}^{\mathrm{eff}}$ of 60\%. Failed sensors are implemented by masking normalized features to zero and appending a binary mask to differentiate dropped inputs from valid zero-mean inputs.

All models share a fixed PPO configuration using CleanRL defaults for continuous-control PPO \citep{huang_cleanrl_2022}, as well as matched architectural capacities (where applicable) and randomly initialized priors. Detailed hyperparameter choices can be found in \ref{sec:hyperparameters}. Agents are trained for one million steps within a single environment instance across 8 random seeds, and minibatches for sequence-based agents comprise trajectory segments of length 16 for truncated backpropagation through time. For models dependent on hidden states, an initial burn-in of 8 timesteps was used to initialize hidden states prior to each gradient computation. For the Transformer- and UniTS-based agents, encoder dropout is enabled during training for stability. Performance during training was measured as the average episodic return across the batch collected per epoch, and the final training curves shown in Figures \ref{fig:training_curves_halfcheetah} and \ref{fig:training_curves_appendix} were generated by plotting the smoothed medians and inter-quartile ranges of the aggregated results across all seeds.

\begin{figure}[tbp!]
    \centering
    \includegraphics[width=\linewidth]{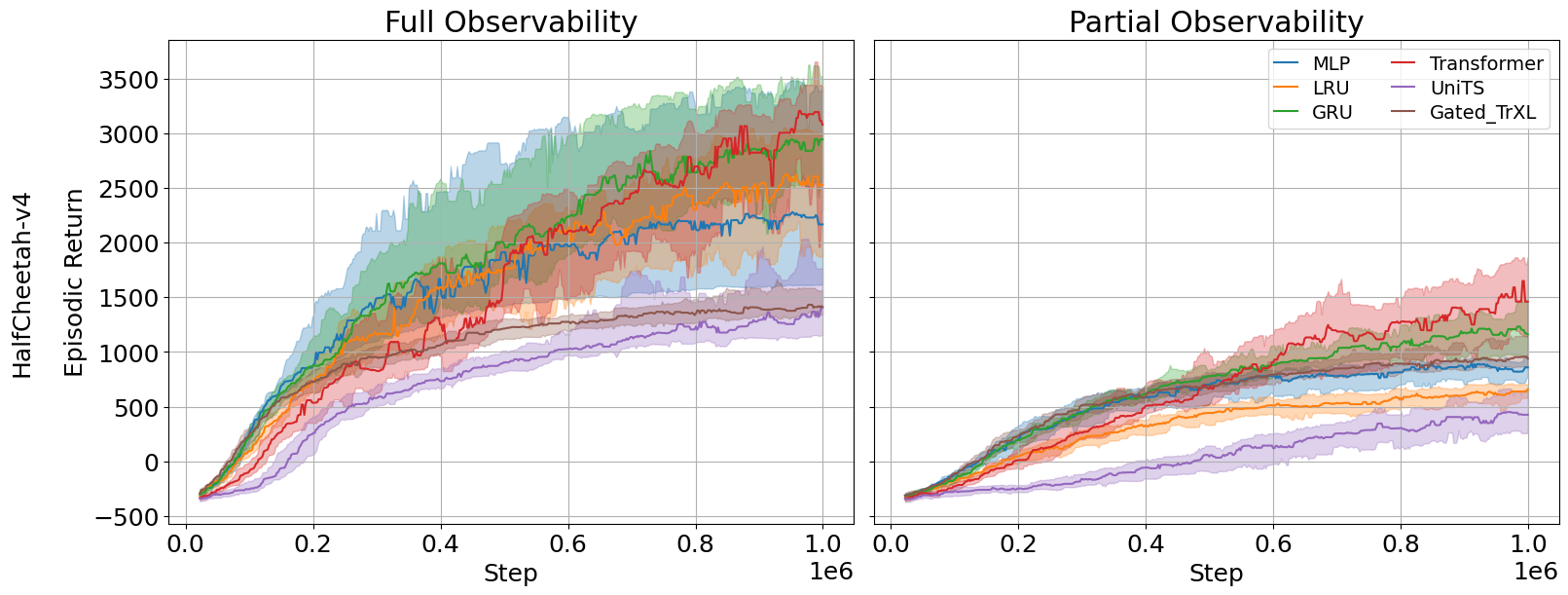}
    \caption{Sample PPO training curves on HalfCheetah-v4 under full (left) and ~60\% partial (right) observability. Lines represent median episodic return and shaded regions denote inter-quartile ranges across 8 random seeds. Training curves generated under partial observability rise more slowly and plateau at lower returns than those produced using fully observed states.}
    \label{fig:training_curves_halfcheetah}
\end{figure}

For evaluation, episodic returns were computed over 100 episodes across all eight seeds using a deterministic version of the learned policy, where actions were taken as the policy mean rather than sampled. The evaluation results per model were then aggregated, and the median performance with 95\% confidence intervals was estimated via bootstrap resampling. Evaluation results are presented in Figure \ref{fig:eval_violins}.

Because MuJoCo environments are mostly Markovian, agents do not require much memory to achieve high rewards under full observability, as each observation contains sufficient information for optimal control. The MLP frequently achieves the highest returns under full observability, benefiting from its architectural simplicity and straightforward feature transformations. Performance among sequence-based policies in the fully observed regime proves more task dependent. On HalfCheetah, multiple sequence models are competitive: the Transformer achieves the strongest performance, and simpler recurrent models like the GRU and LRU also perform well, suggesting that both explicit and latent memory can be beneficial in certain cases even when the state is fully observed. However, on more challenging environments, sequence models underperform to varying degrees, indicating that additional temporal structure and architectural complexity can also be a liability when the current observation is already sufficient. Overall, Figure \ref{fig:eval_violins} indicates that providing temporal context under full observability often does not reliably improve returns and fails to outperform the MLP baseline.

\begin{figure}[htbp!]
    \centering
    \includegraphics[width=\linewidth]{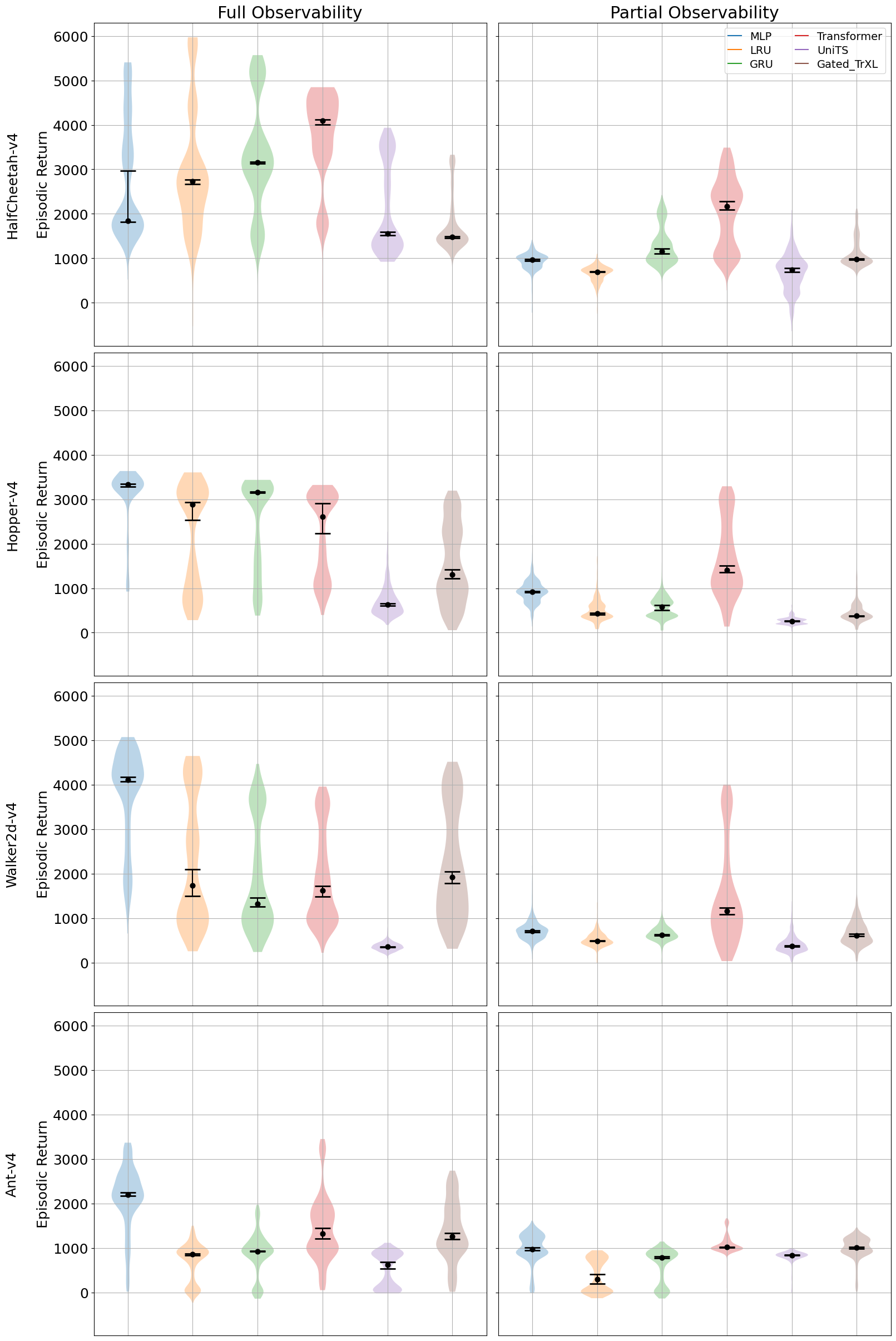}
    \caption{Evaluation episodic returns for PPO agents on MuJoCo environments under full (left) and ~60\% partial (right) observability, with task complexity roughly increasing from top to bottom. Each violin shows the distribution of pooled episodic returns from 100 episodes across 8 random seeds. Black markers denote the median with 95\% bootstrapped CI. While all models suffer performance degradation under partial observability, the Transformer agent demonstrates greater robustness.}
    \label{fig:eval_violins}
\end{figure}

Under the partial observability induced by the sensor failure model, all agents exhibit reduced performance relative to the fully observed setting, confirming that intermittent feature dropout substantially increases task difficulty. Yet, the degree of degradation strongly depends on the architecture and task. Without temporal context, the MLP experiences the most significant drops in performance on average, particularly in Hopper and Walker2d (Figure \ref{fig:eval_violins}). Agents that rely on memory primarily through latent recurrence -- whether implemented via RNNs, state space dynamics, or the gated recurrent attention mechanisms seen in GTrXL -- also demonstrate a substantial lack of robustness. In very few settings, these latent-memory models do outperform the MLP by a small margin, indicating that recurrence can occasionally compensate for missing features, but more often they exhibit noticeably heavier low-return tails. The Transformer policy, on the other hand, proves to be the most robust of all models under partial observability, scoring the highest evaluation medians across all environments and maintaining relatively stable performance. Of the tasks evaluated, Ant proves comparatively challenging under partial observability, and most sequence-based agents do not consistently translate memory over the high-dimensional state into improved returns.

\section{Discussion}
In reinforcement learning settings with missing state information and online PPO, the combination of non-stationarity, partial observability, and the need for flexible memory across variable timescales creates significant architectural challenges. RNNs, SSMs, and even Gated Transformer-XL struggle in this regime because their recurrent dynamics impose strict constraints on memory evolution, treating inputs uniformly and assuming regular input streams. When state information is missing, these assumptions about smoothness and regularity can be violated, causing the recurrent dynamics to diverge or lose critical information, with limited ability to selectively retrieve specific past observations. In contrast, stateless Transformers process all variables jointly within a single sequence and leverage self-attention mechanisms that allow each output to directly attend to all available past tokens, effectively learning temporal correlations by having state dimensions participate in the attention computation. This architecture provides inherent robustness to missing data: the self-attention mechanism can dynamically identify and utilize whichever past observations are present while naturally skipping over gaps, enabling Transformers to maintain performance even under irregular observability by flexibly referencing earlier relevant inputs without being constrained by assumptions of temporal regularity. Our experiments support these points.

A notable exception among the sequence models is UniTS, which underperforms across all settings. We hypothesize that this result is due largely to an inductive bias mismatch. While the Transformer processes all variables jointly, UniTS processes each variable independently during sequence attention, and cross-variable attention happens separately. Though this factorization is beneficial for creating unified representations of diverse time series, it may hinder learning of joint temporal patterns in continuous control RL, since cross-variable interactions are largely deferred to variable-level attention.

\section{Conclusion}
In this work, we studied the robustness of Proximal Policy Optimization under temporally persistent sensor failures, framing observation drift as partial observability with structured temporal correlations. By augmenting PPO with sequence-based policy architectures, we showed that agents can leverage history to infer missing information and maintain performance when observations are unreliable.

We derived high-probability bounds on infinite-horizon reward degradation under a stochastic sensor failure model, clarifying how robustness depends on policy smoothness, critic sensitivity, sensor availability, and failure persistence. Empirically, experiments on MuJoCo benchmarks confirm these insights: while sequence models offer limited benefit under full observability, they substantially improve robustness under sensor dropout. Stateless Transformer-based PPO agents in particular consistently outperform MLP, RNN, and state-space baselines, achieving higher returns and more stable behavior when large fractions of sensors are unavailable.

Overall, our results demonstrate that temporal sequence modeling provides a principled and effective mechanism for robust reinforcement learning in unreliable environments, highlighting attention-based architectures as a promising direction for real-world deployment under observation drift.

\section*{Acknowledgments}
The authors acknowledge the MIT Lincoln Laboratory Supercomputing Center for providing the high-performance computing resources that have contributed to the research results reported in this paper.

DISTRIBUTION STATEMENT A. Approved for public release. Distribution is unlimited.
This material is based upon work supported by the Under Secretary of War for Research and Engineering under Air Force Contract No. FA8702-15-D-0001 or FA8702-25-D-B002. Any opinions, findings, conclusions or recommendations expressed in this material are those of the author(s) and do not necessarily reflect the views of the Under Secretary of War for Research and Engineering.
© 2026 Massachusetts Institute of Technology.

Delivered to the U.S. Government with Unlimited Rights, as defined in DFARS Part 252.227-7013 or 7014 (Feb 2014). Notwithstanding any copyright notice, U.S. Government rights in this work are defined by DFARS 252.227-7013 or DFARS 252.227-7014 as detailed above. Use of this work other than as specifically authorized by the U.S. Government may violate any copyrights that exist in this work.

\bibliography{iclr2026_conference}
\bibliographystyle{iclr2026_conference}

\appendix
\section{Appendix}

\subsection{Proof of Theorem \ref{thm:main}}
The proof of Theorem \ref{thm:main} is based on two lemmas.
\begin{lemma}[Pointwise Wasserstein bound on the per-step loss]\label{lem:KR}
Under Assumptions~\ref{assump:2:wasserstein}–\ref{assump:3:Qlip},
\[
|\Delta_t|
\;=\;
\big|\mathbb E_{a\sim\pi(\cdot\mid h(S_t))}Q^\pi(S_t,a)
-\mathbb E_{a\sim\pi(\cdot\mid h_{M_t}(S_t))}Q^\pi(S_t,a)\big|
\;\le\;
L_Q\,W_1\!\big(\pi(\cdot\mid h(S_t)),\pi(\cdot\mid h_{M_t}(S_t))\big)
\]
\[
\le\ L_Q L_\pi \,\|h(S_t)-h_{M_t}(S_t)\|_1
\;=\;
L_Q L_\pi \sum_{i=1}^d (1-M_{t,i})\,|h_i(S_t)|
\;\le\; C_{\max}.
\]
\emph{Proof.} Kantorovich–Rubinstein duality gives
$|\mathbb E_{\nu}f-\mathbb E_{\nu'}f|\le \mathrm{Lip}(f)\,W_1(\nu,\nu')$ for 1-Lipschitz $f$;
here $f(a)=Q^\pi(S_t,a)$ has $\mathrm{Lip}(f)\le L_Q$ by Assumption~\ref{assump:3:Qlip}. Apply Assumption~\ref{assump:2:wasserstein} and the mask identity $\|h-h_M\|_1=\sum_i(1-M_i)|h_i|$, then Assumption~\ref{assump:1:bounded}. \qed
\end{lemma}
The next lemma states a Markov-chain Bernstein-type concentration inequality.
\begin{lemma}\cite{Paulin:2015}[Bernstein for geometrically ergodic chains]\label{lem:bernstein}
Let $\Xi_t$ satisfy Assumption~\ref{assump:4:augmix}. Let $\{f_t\}_{t\ge 0}$ be bounded functions with $|f_t(\xi)|\le b_t$ and define $S_n=\sum_{t=0}^{n-1}f_t(\Xi_t)$ and $Z_n=S_n-\mathbb E S_n$.
Then, for any $u>0$,
\[
\mathbb P\!\left(Z_n\ge u\right)\ \le\ \exp\!\left(
-\frac{u^2}{\,2\tau\sum_{t=0}^{n-1}b_t^2+\frac{4}{3}b_{\max}\tau\,u}\right),
\quad b_{\max}:=\max_{0\le t\le n-1}b_t.
\]
\end{lemma}

\begin{proof}[Proof of Theorem \ref{thm:main}]

\smallskip
\noindent\textit{Step 0: deterministic pointwise bound.}
By Lemma~\ref{lem:KR}, $X_t=|\Delta_t|\le C_{\max}$ and hence $0\le \gamma^t X_t\le \gamma^t C_{\max}$ for all $t$. This also implies the upper cap bound $S \le \frac{C_{max}}{1-\gamma}$. Trivially, $S\leq \mu_S + \frac{C_{max}}{1-\gamma}$.

\smallskip
\noindent\textit{Step 1: apply Bernstein to the weighted sequence.}
Let $f_t(\Xi_t):=\gamma^t X_t$. Then $|f_t|\le b_t$ with $b_t:=\gamma^t C_{\max}$ and $b_{\max}=C_{\max}$. Define $S_n:=\sum_{t=0}^{n-1}\gamma^t X_t$ and $Z_n:=S_n-\mathbb E[S_n]$. By Lemma~\ref{lem:bernstein},
\[
\mathbb P\!\left(Z_n\ge u\right)\ \le\ \exp\!\left(
-\frac{u^2}{\,2\tau\sum_{t=0}^{n-1}\gamma^{2t}C_{\max}^2+\frac{4}{3}C_{\max}\tau\,u}\right).
\]

\smallskip
\noindent\textit{Step 2: infinite-horizon limit.}
Since $\sum_{t=0}^{\infty}\gamma^{2t}=\frac{1}{1-\gamma^2}$, letting $n\to\infty$ and using monotone convergence yields
\[
\mathbb P\!\left(S-\mu_S\ge u\right)\ \le\ \exp\!\left(
-\frac{u^2}{\,2\tau\,C_{\max}^2/(1-\gamma^2)+\frac{4}{3}C_{\max}\tau\,u}\right).
\]

\smallskip
\noindent\textit{Step 3: solve the quadratic tail.}
Set the right-hand side to $\delta/2$ and solve the Bernstein quadratic inequality for $u$.
A standard inversion gives the convenient choice
\[
u\ =\ C_{\max}\sqrt{\frac{2\tau}{1-\gamma^2}\,\ln\frac{2}{\delta}}\ +\ \frac{4}{3}\,C_{\max}\,\tau\,\ln\frac{2}{\delta},
\]
which implies $\mathbb P\big(S-\mu_S\ge u\big)\le \delta/2$ (an analogous lower-tail bound also holds, but we only need the upper tail).

\smallskip
\noindent\textit{Step 4: bound the mean via independence.}
By Lemma~\ref{lem:KR},
\[
\mu_S\ =\ \mathbb E\!\left[\sum_{t\ge 0}\gamma^t X_t\right]
\ \le\ \frac{L_Q L_\pi}{1-\gamma}\sum_{t\ge 0}\gamma^t\,
\mathbb E\!\left[\sum_{i=1}^d (1-M_{t,i})\,|h_i(S_t)|\right].
\]
Under stationarity and Assumption~\ref{assump:5:indep},
$\mathbb E[(1-M_{t,i})\,|h_i(S_t)|]=(1-\pi_{x,i})\,h_i$,
hence $\mu_S\le \frac{L_Q L_\pi}{1-\gamma}\sum_{i=1}^d (1-\pi_{x,i})\,h_i$.
Combining Steps 1–3 with this mean bound yields the stated inequality.
\end{proof}

\begin{remark}[Signed version]
Defining $S_\Delta:=\sum_{t\ge 0}\gamma^t \Delta_t$ and repeating the proof with $\Delta_t$ (instead of $|\Delta_t|$) gives the same deviation term and a two-sided bound for $S_\Delta$ about its mean. Since $|S_\Delta|\le S$, the result in Thm. \ref{thm:main} is a conservative degradation guarantee.
\end{remark}


\subsection{GRU and LRU Specifics}  

\paragraph{GRU (RNN).} \citep{cho2014}
A GRU is obtained by letting $h_t \in \mathbb{R}^{d}$ and defining gate-controlled nonlinear updates:
\begin{align*}
r_t &= \sigma(W_r x_t + U_r h_{t-1}), \\
u_t &= \sigma(W_u x_t + U_u h_{t-1}), \\
\tilde{h}_t &= \tanh(W_h x_t + U_h (r_t \odot h_{t-1})), \\
h_t &= (1-u_t)\odot h_{t-1} + u_t \odot \tilde{h}_t, \\
z_t &= h_t,
\end{align*}
optionally wrapped with residual connections and gating blocks in a stacked encoder.

\paragraph{LRU (diagonal linear SSM).} \citep{orvieto2023}
An LRU-style SSM can be expressed by choosing $h_t$ in a transformed (complex) latent space and using a linear recurrence with diagonal dynamics:
\begin{equation*}
h_t = \Lambda h_{t-1} + B x_t,
\end{equation*}
where $\Lambda$ controls retention/decay (analogous to a learned, state-dependent ``carry'' behavior), and $z_t$ is produced by a learned readout followed by stabilizing blocks (e.g., residual/GLU/normalization) in a stacked encoder.


\subsection{Training Details}

\subsubsection{RNN/SSM Integration}
\paragraph{Episode boundary handling.}
To prevent leakage of information across episodes, we reset memory when an environment terminates:
\begin{equation*}
h_{t-1} \leftarrow (1-d_t)\, h_{t-1},
\end{equation*}
before applying the update in Eq.~\ref{eq:generic_encoder}.
This reset is applied during both rollout collection and training-time unrolling. 


\paragraph{Truncated backpropagation and burn-in.}
During PPO updates, we unroll $\mathcal{E}_\psi$ over short subsequences.
Optionally, we use a \emph{burn-in} prefix to update hidden state without gradients and then compute PPO losses on the subsequent segment using the warmed state:
\begin{equation*}
h_{\tau} \leftarrow \mathcal{E}_\psi(h_0, x_{1:\tau}; d_{1:\tau}) \quad (\text{no grad}), \qquad
\text{loss on } x_{\tau+1:\tau+K}.
\end{equation*}
This is implemented for our recurrent agents to ensure hidden states update alongside model weights each epoch. 

\subsubsection{Hyperparameters and Training Curves}
\label{sec:hyperparameters}

\begin{table}[h!]
    \centering
    \caption{Proximal Policy Optimization (PPO) Hyperparameters}
    \label{tab:ppo_hyperparameters}
    \begin{tabular}{ll}
        \toprule
        \textbf{Hyperparameter} & \textbf{Value} \\
        \midrule
        Total timesteps        & $1 \times 10^{6}$ \\
        Learning rate          & $3 \times 10^{-4}$ \\
        Number of steps        & 2048 \\
        Discount factor ($\gamma$) & 0.99 \\
        GAE parameter ($\lambda$)  & 0.95 \\
        Number of minibatches  & 32 \\
        Update epochs          & 10 \\
        Clipping coefficient   & 0.2 \\
        Entropy coefficient    & 0.0 \\
        Value function coefficient & 0.5 \\
        Maximum gradient norm  & 0.5 \\
        \bottomrule
    \end{tabular}
\end{table}

\begin{table}[h!]
\centering
\caption{Architectural Hyperparameters per Model}
\label{tab:arch_common_extras}
\setlength{\tabcolsep}{5pt}
\begin{tabular}{@{}lccccp{0.38\linewidth}@{}}
\toprule
\textbf{Model} & \textbf{Hidden} & \textbf{Layers} & \textbf{Heads} & \textbf{Dropout} & \textbf{Additional} \\
\midrule
Transformer            & 128 & 2 & 2 & 0.1 & -- \\
Gated TransformerXL    & 128 & 2 & 2 & 0   & Memory length 8 \\
UniTS                  & 128 & 2 & 2 & 0.1 & Patch length 1; Prompt tokens 6 \\
GRU                    & 128 & 4 & --& 0   & -- \\
LRU                    & 128 & 4 & --& 0   & $r_{\min}$ 0.9; $r_{\max}$ 0.999; Max phase 6.28 \\
MLP                    & 128  & --& --& --  & -- \\
\bottomrule
\end{tabular}
\end{table}

\begin{figure}[htbp!]
    \centering
    \includegraphics[width=\linewidth]{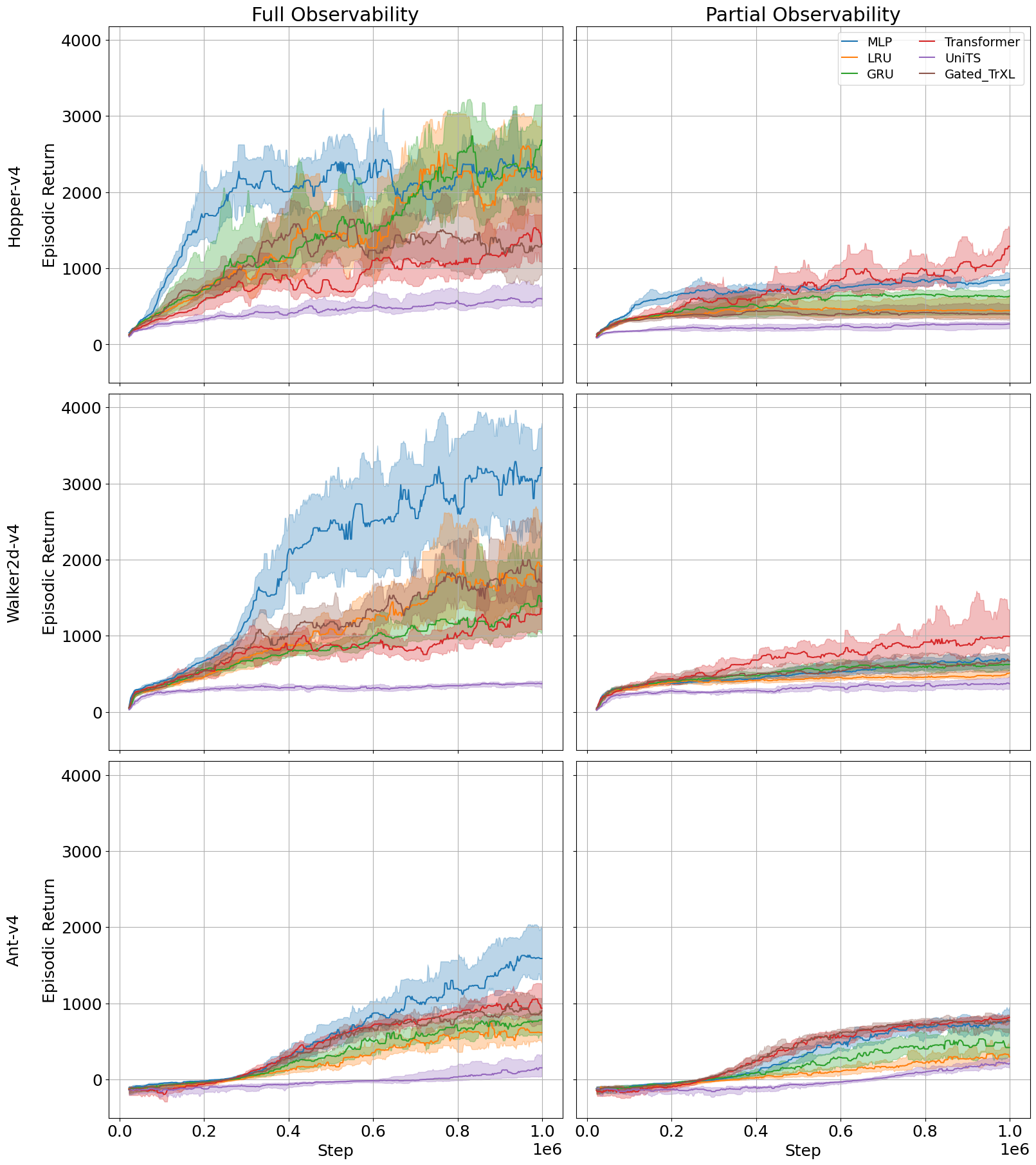}
    \caption{PPO training curves for Hopper-v4, Walker2d-v4, and Ant-v4 under full (left) and ~60\% partial (right) observability. Lines represent median episodic return and shaded regions denote inter-quartile ranges across 8 random seeds.}
    \label{fig:training_curves_appendix}
\end{figure}

\end{document}